\documentclass[letterpaper, journal]{IEEEtran}

\usepackage{cite}
\usepackage{amsmath,amssymb,amsfonts, amsthm}
\usepackage{algorithmic}
\usepackage{graphicx}
\usepackage{textcomp}
\usepackage{xcolor}
\usepackage{subcaption}

\usepackage[ruled,vlined]{algorithm2e}   % main algorithm 
\usepackage{setspace} % For pseudo-code style spacing

\usepackage[colorlinks=true, linkcolor=black, citecolor=blue, urlcolor=blue]{hyperref}

\usepackage{bm}
\usepackage{booktabs} % for tablular midrule bottomrule
\usepackage{mathtools}

\usepackage{multirow}
\usepackage{makecell}
\usepackage{threeparttable}
\usepackage{adjustbox}

\usepackage{titlesec}
\titlespacing*{\subsection}{0pt}{4.0pt}{1.0pt}
\titlespacing*{\section}{0pt}{8.0pt}{4.0pt}

\newtheorem{proposition}{Proposition}

\newcommand{\xobs}{x_{\mathrm{obs}}}
\newcommand{\qnaive}{q_{\mathrm{naive}}}
\newcommand{\qcorr}{q_{\mathrm{bridge}}}
\newcommand{\qrefined}{q_{\mathrm{bridge+}}}
\newcommand{\KL}{D_{\mathrm{KL}}}

\newcommand{\spaceabs}{\,|\,}

\def\BibTeX{{\rm B\kern-.05em{\sc i\kern-.025em b}\kern-.08em
    T\kern-.1667em\lower.7ex\hbox{E}\kern-.125emX}}
\begin{document}

\title{Bridged SBI: Correcting Biased Low-Fidelity Posteriors for Cost-Efficient High-Fidelity Inference\\

\author{Gahee Kim$^{1}$, Yuki Kadokawa$^{1}$, Sandro M. Alcantara Tacora$^{1}$, Taro Abe$^{2}$, \\Daisuke Endo$^{2}$, Genki Yamauchi$^{2}$, Takeshi Hashimoto$^{2}$, and Takamitsu Matsubara$^{1}$%
\thanks{This work was supported in part by the Japan Science and Technology Agency (JST) (Moonshot Research and Development) under Grant JPMJMS2032, and in part by JST (Core Research for Evolution Science and Technology (CREST)) under Grant JPMJCR2555. Corresponding author: Gahee Kim (email: gahee.kim@naist.ac.jp).}
\thanks{$^{1}$G. Kim, Y. Kadokawa, S. M. Alcantara Tacora, and T. Matsubara are with the \textsc{Nara Institute of Science and Technology}, Nara, Japan.}
\thanks{$^{2}$T. Abe, D. Endo, G. Yamauchi, and T. Hashimoto are with the \textsc{Public Works Research Institute}, Ibaraki, Japan.} 
}}

\maketitle

\begin{abstract}
Accurate calibration of particle-based simulators is crucial for robotic earthwork simulation, but analytical calibration is challenging due to this task's highly nonlinear particle dynamics and the black-box nature of conventional simulators. Although simulation-based inference (SBI) can estimate posterior distributions over simulation parameters solely from forward simulations, applying SBI directly to high-fidelity (HF) particle simulators is often computationally prohibitive. Low-fidelity (LF) simulators with coarser particles can reduce this cost, but changes in particle size and particle count shift the parameter values needed to reproduce the same observation, producing biased LF posteriors. We propose Bridged SBI, which leverages a biased but informative LF posterior to guide HF inference. This method first uses inexpensive LF simulations to identify a coarse high-density parameter region, and then it learns a local residual bridge to transport LF posterior samples toward HF-consistent regions by correcting the LF--HF discrepancy. We analyze how sequential multi-fidelity SBI (Naive-MF) can suffer from LF-induced posterior miscoverage when it directly relies on the LF posterior without discrepancy correction. We then show that Bridged SBI is designed to alleviate this issue by explicitly modeling the LF--HF discrepancy through residual correction. Experiments on both sim-to-sim particle-parameter calibration and real-to-sim calibration with real soil observation show that Bridged SBI produces more accurate and reliable HF posteriors than HF-only SBI or the Naive-MF baseline, especially under limited HF simulation costs.
\end{abstract}

% \end{abstract}
\begin{IEEEkeywords}
Calibration and Identification, Probability and Statistical Methods
\end{IEEEkeywords}

\section{Introduction}
\IEEEPARstart{C}{alibrating} particle-based simulators is essential for ensuring accurate digital twins of robotic earthwork tasks, but this is challenging due to the highly nonlinear and black-box nature of particle dynamics. In particle models, including the discrete element method (DEM), macroscopic soil behavior emerges from particle-level interactions governed by particle parameters such as friction and restitution coefficients~\cite{CUNDALL1979}. Therefore, these parameters must be estimated to reproduce soil behavior in simulation. However, analytical identification is typically intractable, and manual tuning is labor-intensive~\cite{COETZEE2017104}. To address this challenge, simulation-based inference (SBI)~\cite{Cranmer2020-cc}, which estimates posterior distributions over simulation parameters using only forward simulations, has recently emerged as a promising approach.

\begin{figure}[tb]
\centering 
\includegraphics[width=0.95\columnwidth]{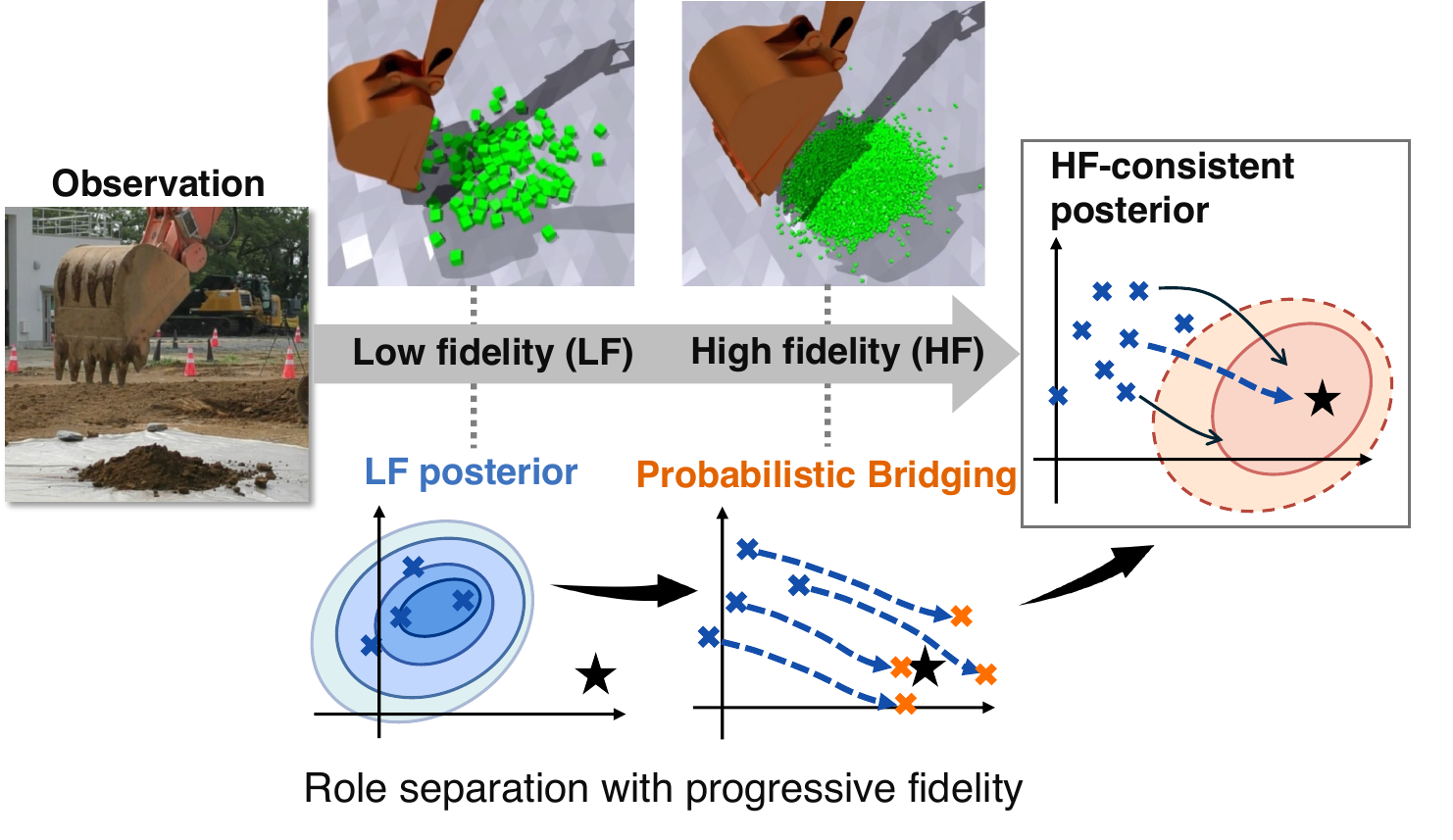}
\captionsetup{font=footnotesize}
\caption{
Overview of Bridged SBI for low-fidelity-guided posterior correction. Bridged SBI decomposes simulation calibration into low-fidelity-guided exploration and posterior correction. Low-fidelity simulator provides a low-cost coarse posterior over plausible but potentially biased parameter regions due to cross-fidelity model discrepancy. Bridged SBI then uses a small number of high-fidelity simulations to learn a local residual bridge. This bridge applies learned residual corrections to coarse posterior samples and transports them toward an HF-consistent posterior. }
\label{fig:main_fig1}
\vspace{-10pt}
\end{figure}

Despite its flexibility, applying SBI to high-fidelity (HF) particle simulators is computationally prohibitive, since reducing particle size to increase fidelity dramatically increases simulation costs. Low-fidelity (LF) simulators with larger particles can mitigate this cost burden. However, the resulting posteriors generally deviate from the HF target posterior because changes in particle size and count alter fundamental dynamics such as collision frequency and energy dissipation patterns~\cite{Bierwisch2009, dufresne2022energy}. Consequently, the parameter regions needed to reproduce a given observation can shift across fidelities~\cite{ROESSLER201858}, posing a trade-off between cost and estimation accuracy.

To address this trade-off, we exploit the observation that the LF simulator often provides a systematically biased but computationally cheaper representation of the HF simulator. Although changing the simulation fidelity shifts the corresponding parameter regions, the LF and HF simulators still describe the same underlying physical system, and their simulated outputs thus remain structurally correlated across fidelities. This suggests that LF simulations may still provide useful structural information about the HF posterior, even when the LF posterior itself is biased or misaligned. Consequently, rather than treating the LF posterior as a reliable approximation of the HF posterior, we instead use it as an imperfect but informative guide that can be corrected with a small number of HF simulations.

In this paper, we propose \emph{Bridged SBI}, which decomposes HF inference into LF inference and LF--HF discrepancy correction (Fig.~\ref{fig:main_fig1}). The proposed method first uses inexpensive LF simulations to obtain a coarse estimate and then applies a limited number of HF simulations near the resulting coarse parameter region to learn a local residual bridge. This bridge transports LF posterior samples toward HF-consistent regions, permitting HF inference with substantially fewer HF simulations. We also provide a theoretical analysis showing that sequential multi-fidelity SBI (Naive-MF), which ignores the LF--HF discrepancy, can suffer from LF-induced posterior miscoverage, whereas Bridged SBI can mitigate this issue by explicitly modeling the LF--HF discrepancy through residual correction, allowing recovery of HF-consistent posterior regions that may be underrepresented by the LF posterior. We validate Bridged SBI through sim-to-sim and real-to-sim particle-parameter calibration tasks, demonstrating its effectiveness in a robotic earthwork scenario using real-world soil observation. 

The main contributions of this paper are as follows:
\begin{itemize}
\item We formulate particle-parameter calibration as decomposed posterior inference, combining LF inference with explicit LF--HF discrepancy correction. We analyze how LF-induced posterior miscoverage can arise in multi-fidelity SBI when the LF--HF discrepancy is not explicitly addressed.
\item We propose \emph{Bridged SBI}, which uses LF simulations to explore low-cost parameter-space and limited HF simulations to learn a local residual bridge for correcting LF posterior samples.
\item We demonstrate through sim-to-sim and real-to-sim calibration tasks that Bridged SBI improves the accuracy and stability of HF inference, particularly under severely constrained HF simulation budgets, reducing KL divergence by up to 75\% compared with Naive-MF and up to 67\% compared with HF-only SBI in real-to-sim calibration.
\end{itemize}

\section{Related Work}
\subsection{Parameter Estimation in Particle Simulators}

Particle-based simulators, such as those adopting the discrete element method (DEM), are widely used to model granular materials~\cite{CUNDALL1979} and robotic earthwork tasks~\cite{kadokawa2025}. Their behavior is governed by particle-level parameters including friction and restitution coefficients, which must be calibrated to reproduce target macroscopic behaviors~\cite{COETZEE2017104}. Existing calibration approaches often rely on matching handcrafted macroscopic statistics, such as the angle of repose or pile shape, between simulation and experiments~\cite{COETZEE2017104,ROESSLER201858}. More recently, simulation-based inference (SBI) has been explored in robotics for simulator calibration~\cite{Ramos2019-oy, gahee2025-sbi} and domain randomization~\cite{Muratore2022-cy}. In particular, neural SBI methods enable posterior inference directly from high-dimensional observations such as images~\cite{Papamakarios2016-ki,Lueckmann2017-sa,Greenberg2019-on,Cranmer2020-cc}, making them suitable for robotics earthwork scenarios involving complex spatial deformation caused by robot--soil interaction. However, applying SBI to high-fidelity particle simulators remains computationally challenging because increasing simulation fidelity typically requires the reduction of particle size, which substantially increases the number of particles as well as simulation cost~\cite{DiRenzo2021CoarseGrainReview}.

To mitigate this challenge, this paper considers parameter estimation from a multi-fidelity perspective. Instead of relying solely on expensive high-fidelity (HF) simulations, we exploit low-fidelity (LF) simulations, which use larger particles and fewer total particles, as a cost-efficient source of information for inference. By combining LF simulations with a limited number of HF simulations, we reduce the simulation burden of HF-only SBI. Since this process relies only on forward simulations and does not require an explicit likelihood function or analytical model, it remains within the SBI framework and can be applied broadly to black-box simulators.

\subsection{Multi-Fidelity Simulation-Based Inference}

In the context of multi-fidelity SBI, Approximate Bayesian Computation (ABC) methods have used LF simulations to narrow the sampling space for HF simulations~\cite{Warne2018-gk,Prescott2024-re}. Recent neural SBI approaches have incorporated LF information into the training of HF neural posterior estimators, for example by using LF posteriors to guide the allocation of HF simulations~\cite{Thiele2025-zm, Hikida2025-ah} or by initializing HF posterior estimators from LF-pretrained models~\cite{Krouglova2026-ew}. While these methods reduce the required number of expensive HF simulations, they implicitly rely on the ability of LF-guided parameter regions to sufficiently cover HF-plausible regions. However, this assumption can fail in particle-based simulators, where changes in particle size and count can affect particle dynamics and the corresponding parameter regions across fidelities~\cite{Bierwisch2009,ROESSLER201858}. Consequently, the LF posterior can be biased and assign little probability mass to regions that remain plausible under the HF simulator.

Unlike existing multi-fidelity neural SBI methods, our approach explicitly addresses the LF--HF discrepancy by learning a correction from LF-guided to HF-consistent parameter regions, thus allowing recovery of HF-plausible regions that may be underrepresented by the LF posterior.

\section{Proposed Method}
In this section, we first show that sequential multi-fidelity SBI (Naive-MF) can fail when LF--HF discrepancy is ignored. We then introduce Bridged SBI, which decomposes HF inference into two stages: LF inference and the subsequent discrepancy correction. This correction is achieved as a local residual bridge learned from a small number of HF simulations.

\subsection{Problem Setting}

Let $\theta \in \Theta$ denote simulator parameters (e.g., friction, restitution) and $x \in \mathcal{X}$ denote simulator outputs (e.g., height maps, trajectories). We consider two stochastic black-box simulators defined on the same parameter space $\Theta$ and output space $\mathcal{X}$:
 (1) \textbf{High-fidelity (HF):} $x_h \sim p_h(x \spaceabs \theta)$, which more accurately represents the target behavior but is computationally expensive; and (2) \textbf{Low-fidelity (LF):} $x_l \sim p_l(x \spaceabs \theta)$, which is computationally cheaper but less accurate. Both are black-box simulators, so $p_h(x \spaceabs \theta)$ and $p_l(x \spaceabs \theta)$ are not available in closed form and can only be accessed through forward simulations.

Given a real-world observation $\xobs$, our goal is to infer the HF posterior, which we refer to as the target posterior,
\begin{equation}
p_h^{\star}(\theta \spaceabs \xobs)
\propto
p(\theta) L_h(\xobs \spaceabs \theta),
\end{equation}
where $p(\theta)$ is the prior and $L_h(\xobs \spaceabs \theta)$ denotes the HF-induced likelihood, which is intractable due to the simulator's black-box nature. This posterior is estimated by SBI from a large amount of simulation data, which is often impractical for costly HF simulators, motivating a multi-fidelity strategy that combines LF and HF simulations.

\subsection{Structural Limitation of Sequential Multi-Fidelity Inference}\label{sec:limit_of_naive}

A common heuristic in multi-fidelity SBI is to first estimate an LF posterior using inexpensive LF simulations and then use it as an informed prior for the HF stage. We denote the LF posterior by
\begin{equation}
p_l^{\star}(\theta \spaceabs \xobs)
\propto 
p(\theta)\, L_l(\xobs \spaceabs \theta),
\end{equation}
where $L_l(\xobs \spaceabs \theta)$ denotes the likelihood induced by the LF simulator. A sequential multi-fidelity SBI, which we denote as Naive-MF, can be written conceptually as
\begin{equation}\label{eq:naive-mf}
\qnaive (\theta \spaceabs x_{\mathrm{obs}}) \propto 
p_l^{\star}(\theta \spaceabs x_{\mathrm{obs}})\,L_h(x_{\mathrm{obs}} \spaceabs \theta).
\end{equation}
While Naive-MF can improve the overall time and cost efficiency of simulation by using inexpensive LF simulations to guide expensive HF inference, it can fail due to LF--HF discrepancies, as formalized by the following result. Here, we let $Z$ be the normalization constant in (\ref{eq:naive-mf}).

\begin{proposition}[Irreducible Bias under LF Miscoverage]
\label{proposition:naive}
Let $\qnaive$ be defined by (\ref{eq:naive-mf}), and define the KL divergence as
\[
E_{\mathrm{naive}}
=
\KL\!\left(
p_h^\star(\cdot \spaceabs \xobs)
\,\|\,
q_{\mathrm{naive}}(\cdot \spaceabs \xobs)
\right).
\]
Assume there exists a measurable set $A \subset \Theta$ such that
\[
p_h^\star(A \spaceabs \xobs)=\rho>0,
\qquad
p_l^\star(A \spaceabs \xobs)=\varepsilon>0,
\]
with $\varepsilon \ll \rho$, and further assume that $L_h(\xobs \spaceabs \theta)$ is upper bounded by $M$ on $A$.
Then, for the finite constant $C:=Z^{-1}M$, we have
\begin{gather}
q_{\mathrm{naive}}(A \spaceabs \xobs)
\leq C\varepsilon, \ \ \ \text{and}\\
E_{\mathrm{naive}}
\geq
\rho \log \frac{\rho}{C\varepsilon}
+
(1-\rho)\log(1-\rho).
\end{gather}
\end{proposition}
\begin{proof}
The proof is given in Appendix~A.
\end{proof}
This proposition shows that when the LF posterior underestimates an HF-plausible region, i.e., $\epsilon \ll \rho$, directly using such an LF posterior for sequential HF refinement constrains the resulting posterior on this region as $q_{\mathrm{naive}}(A \spaceabs \xobs) \leq C\epsilon$. As a result, Naive-MF cannot sufficiently recover HF-plausible probability mass once it is severely underrepresented by the LF posterior. We refer to this failure mode as LF-induced posterior miscoverage. This motivates explicit LF--HF discrepancy correction when using the LF posterior to guide HF inference.

\subsection{Bridged SBI}
\label{sec:def_bamsi_corr}
Bridged SBI explicitly separates the roles of LF and HF simulations. 
LF simulations are used to infer a coarse posterior over parameter regions, whereas limited HF simulations are used to learn a local correction that transports LF posterior samples to align with the HF simulation. We express the target posterior through an ideal correction
kernel over the LF latent variable $z_l=(\theta_l,x_l)$, where
$x_l \sim p_l(x\mid \theta_l)$:
\begin{equation}\label{eq:Target_decom}
p_h^\star(\theta_h \spaceabs \xobs ) 
= 
\int K^\star(\theta_h \spaceabs z_l, \xobs) \,
p_l(z_l \spaceabs \xobs) 
\, dz_l .
\end{equation}
Here, $K^\star(\theta_h \spaceabs z_l, \xobs)$ denotes the ideal
conditional correction distribution that transports an LF latent variable
$z_l$ to an HF-consistent parameter $\theta_h$, given the observation
$\xobs$. In Bridged SBI, we approximate this correction distribution using a parametric model $q_\phi(\theta_h \spaceabs z_l, x)$, where $x=x_h$ during training and $x=x_{\mathrm{obs}}$ at inference. 
\begin{equation}\label{eq:BAMSI_Corr}
q_{\mathrm{bridge}}(\theta_h \spaceabs x_{\mathrm{obs}})
:= 
\int
q_\phi(\theta_h \spaceabs z_l, x_{\mathrm{obs}})
\, p_l(z_l \spaceabs x_{\mathrm{obs}})
\, dz_l .
\end{equation}

We refer to $q_\phi$ as the probabilistic bridge model. 
A key component of Bridged SBI is to parameterize this bridge as an additive residual correction. 
Specifically, instead of modeling the HF parameter $\theta_h$ directly, we learn a conditional residual density 
$q_{\phi,\Delta}(\Delta \mid z_l,x)$ with $\Delta=\theta_h-\theta_l$ and obtain corrected samples by $\theta_h=\theta_l+\Delta$. 
This formulation trains a correction centered at LF-guided samples, making the bridge easier to learn when only limited HF data are available~\cite{he2016deep}. 
Details of the training and sampling process are given in Section~\ref{sec:bamsi_imp}. 

The following results clarify the theoretical role of this bridge correction.

\begin{proposition}[Bridge-Controlled Recovery Error]
\label{prop:bamsi_recovery}
Let $p_h^\star$ and $\qcorr$ be defined by
(\ref{eq:Target_decom}) and (\ref{eq:BAMSI_Corr}), respectively.
Define the KL divergence as 
\[
E_{\mathrm{bridge}}
=
\KL\!\left(
p_h^\star(\cdot \spaceabs \xobs)
\,\|\, 
q_{\mathrm{bridge}}(\cdot \spaceabs \xobs)
\right).
\]
Then
\begin{equation}
E_{\mathrm{bridge}}
\leq
\mathbb{E}_{p_l(z_l|\xobs)}
\!\left[
\KL\!\left(
K^\star(\cdot \spaceabs z_l,\xobs)
\,\|\,
q_\phi(\cdot \spaceabs z_l,\xobs)
\right)
\right].
\label{eq:bamsi_kl_bound}
\end{equation}
Furthermore, for any $A \subset \Theta$ with
$ p_h^\star(A \spaceabs \xobs)=\rho>0,$ if the average bridge error on $A$ satisfies
\begin{equation}
\mathbb{E}_{p_l(z_l \spaceabs \xobs)}
\left[
\left|
q_\phi(A \spaceabs z_l,\xobs)
-
K^\star(A \spaceabs z_l,\xobs)
\right|
\right]
\leq \eta ,
\label{eq:bamsi_set_error}
\end{equation}
then,
\begin{equation}
\qcorr (A \spaceabs \xobs)
\geq
\rho-\eta .
\label{eq:bamsi_recovery_lower}
\end{equation}
\end{proposition}
\begin{proof}
The proof is given in Appendix~A.
\end{proof}

Thus, unlike Naive-MF, whose mass on $A$ is constrained by the LF posterior mass, Bridged SBI can recover the HF posterior mass on $A$ up to the average conditional bridge approximation error $\eta$.

\subsection{Residual Bridge Learning and Posterior Sampling}\label{sec:bamsi_imp}

This section describes how $\qcorr$ in (\ref{eq:BAMSI_Corr}) is implemented using an additive residual bridge, how the bridge is trained online with limited HF simulations, and how LF posterior samples are corrected. 
The overall procedure is summarized in Algorithm~\ref{alg:bamsi}. 
We also describe an optional HF refinement step that constructs $\qrefined$, which further sharpens $\qcorr$ when additional HF budget is available.

\vspace{5pt}
\noindent \textbf{Training the LF estimator:}
We first train neural SBI to obtain an LF posterior $\hat{p}_l(\theta\spaceabs x)$. In this paper, we use neural posterior estimation (NPE)~\cite{Papamakarios2016-ki, Lueckmann2017-sa, Greenberg2019-on}.

\vspace{5pt}
\noindent \textbf{Residual bridge training:}
We collect a small number of HF simulations around the LF-guided parameter region.
HF parameters are drawn from a sampling distribution $r_h(\theta\spaceabs\xobs)$ constructed from the LF posterior  $\hat{p}_l(\theta\spaceabs\xobs)$, for example, by perturbing or interpolating its samples. The simplest choice is to set $r_h(\theta\spaceabs\xobs)=\hat{p}_l(\theta\spaceabs\xobs)$. For each  $\theta_h\sim r_h(\theta\spaceabs\xobs)$, we simulate $x_h\sim p_h(x\spaceabs\theta_h)$ and pair it with LF variables 
$z_l=(\theta_l,x_l)$ resampled from the LF dataset according to weights induced by $\hat{p}_l(\theta_l\spaceabs\xobs)$. 

The bridge is parameterized by a neural conditional density estimator and trained as 
\begin{equation}
q_\phi(\theta_h\spaceabs z_l,x)
=
q_{\phi,\Delta}(\theta_h-\theta_l\spaceabs z_l,x),
\end{equation}
where $x=x_h$ during training and $x=\xobs$ during inference. This formulation is not tied to a specific density model: Any neural network capable of approximating the conditional density can be used.  The training data are collected conditional on $\xobs$, so this residual model learns the local LF--HF residual structure from LF-consistent to HF-consistent samples around $\xobs$.

\vspace{5pt}
\noindent \textbf{Inference with Bridged SBI:}
Sampling from $\qcorr$ follows (\ref{eq:BAMSI_Corr}). 
In our implementation, we use the Dirac observation projection
$p(x_l\spaceabs\xobs):=\delta(x_l-\xobs)$. 
Under this projection, the LF latent variable $z_l\sim p_l(z_l\spaceabs\xobs)$ is sampled by setting
\[
x_l=\xobs,
\qquad
\theta_l\sim \hat{p}_l(\theta_l\spaceabs\xobs).
\]
The bridge model then maps this LF posterior sample toward an HF-consistent sample by applying a residual correction
\[
\Delta \sim q_{\phi,\Delta}(\Delta\spaceabs \theta_l,x_l,\xobs),
\qquad
\theta_{\mathrm{corr}} :=\theta_l+\Delta .
\]
Here, $\theta_{\mathrm{corr}}$ denotes the corrected parameter sample used as
a sample from the bridged posterior $q_{\mathrm{bridge}}(\theta \mid x_{\mathrm{obs}})$. 

The Dirac projection used here is a simple deterministic choice. More generally, the LF latent distribution $p_l(z_l\spaceabs\xobs)$ can be  approximated as
\begin{equation}
p_l(\theta_l\spaceabs x_l,\xobs)\,p(x_l\spaceabs\xobs)
\approx
\hat{p}_l(\theta_l\spaceabs x_l)\,p(x_l\spaceabs\xobs),
\end{equation}
where we use the independence assumption $\theta_l \perp\!\!\!\perp \xobs \spaceabs x_l$. If there is available domain knowledge on how LF and HF observations differ, this extension can improve robustness at the observation level, beyond parameter-level correction alone. In this paper, we use the Dirac projection, but other choices, such as a spike-and-slab distribution~\cite{Ward2022-fq}, could also be used.

\vspace{5pt}
\noindent \textbf{Optional HF refinement:}
When additional HF simulations are available, we can further refine $\qcorr$ by treating it as an informed prior:
\begin{equation}
\qrefined(\theta\spaceabs\xobs)
\propto
\qcorr(\theta\spaceabs\xobs)L_h(\xobs\spaceabs\theta).
\end{equation}
In practice, $\qrefined$ is obtained by training an additional NPE on HF simulations with parameters sampled from $\qcorr(\theta\spaceabs\xobs)$, and equivalently by using the corrected samples $\theta_{\mathrm{corr}}$. This optional refinement step further sharpens the corrected posterior after the bridge correction has mitigated LF-induced miscoverage.

\section{Sim-to-Sim Experiment}\label{sec:s2s}
\begin{algorithm}[tb]
\footnotesize
\caption{Pipeline of Bridged SBI}
\label{alg:bamsi}
\DontPrintSemicolon

\KwIn{Prior $p(\theta)$, Observation $\xobs$, Simulators $p_l(x|\theta), p_h(x|\theta)$, Budgets $(N_l,N_h,N_k)$}
\KwOut{Corrected samples $\{\theta_{\mathrm{corr}}^{(m)}\}_{m=1}^{M}$, Optional refined posterior $\qrefined$}

\BlankLine
\tcp{\textbf{Stage 1: LF Posterior Estimation}}
Generate LF dataset $\mathcal{D}_l=\{(\theta_l^{(i)},x_l^{(i)})\}_{i=1}^{N_l}$ where $\theta_l^{(i)}\sim p(\theta)$ and $x_l^{(i)}\sim p_l(x|\theta_l^{(i)})$\;
Train LF posterior estimator $\hat p_l(\theta|x)$ on $\mathcal{D}_l$ via NPE\;

\BlankLine
\tcp{\textbf{Stage 2-1: LF-Guided HF Sampling}}
Construct local HF proposal $r_h(\theta|\xobs)$ from $\hat p_l(\theta|\xobs)$\;
Collect limited HF samples $\mathcal{D}_h = \{(\theta_h^{(j)},x_h^{(j)})\}_{j=1}^{N_h}$ where $\theta_h^{(j)}\sim r_h(\theta|\xobs)$ and $x_h^{(j)}\sim p_h(x|\theta_h^{(j)})$\;

\BlankLine
\tcp{\textbf{Stage 2-2: LF-HF Bridge Dataset Collection}}
Resample $\tilde{\mathcal{D}}_l=\{(\theta_l^{(k)},x_l^{(k)})\}_{k=1}^{N_k}$ from $\mathcal{D}_l$ using weights from $\hat p_l(\theta_l|\xobs)$\;
Form mixed pairs $\mathcal{P} = \{(\theta_l^{(k)},x_l^{(k)},x_h^{(j)},\Delta^{(j,k)})\}$ where $\Delta^{(j,k)}=\theta_h^{(j)}-\theta_l^{(k)}$\;

\BlankLine
\tcp{\textbf{Stage 2-3: Residual Bridge Training}}
Train a neural conditional density estimator $q_{\phi,\Delta}$ to minimize: 
$\mathcal{L}_{\mathrm{bridge}}(\phi) = -\frac{1}{|\mathcal{P}|} \sum_{\mathcal{P}} \log q_{\phi,\Delta}(\Delta|\theta_l,x_l,x_h)$\;

\BlankLine
\tcp{\textbf{Stage 3: Inference with Bridged SBI}}
\For{$m=1$ \KwTo $M$}{
  Sample LF parameter $\theta_l^{(m)}\sim\hat p_l(\theta|\xobs)$ and set $x_l^{(m)} \leftarrow \xobs$ \tcp*{Dirac proj.}
  Predict residual $\Delta^{(m)} \sim q_{\phi,\Delta}(\Delta|\theta_l^{(m)},x_l^{(m)},\xobs)$\;
  Apply correction $\theta_{\mathrm{corr}}^{(m)} = \theta_l^{(m)}+\Delta^{(m)}$\;
}

\BlankLine
\tcp{\textbf{[Optional] HF Refinement}}
Train HF estimator $\qrefined$ on sampling distribution $\qcorr$.\;

\BlankLine
\Return{$\{\theta_{\mathrm{corr}}^{(m)}\}_{m=1}^{M}$ and $\qrefined$}
\end{algorithm}

We used a robot simulator equipped with a particle model. With earthwork applications in mind, we developed a simulation with an excavator model (Fig.~\ref{fig:excavator_fidelity}).

\subsection{Simulation Model}

We simulate a particle-pouring task with an earthwork excavator in Isaac Gym~\cite{Makoviychuk2021IsaacGym}. In the LF simulator, the excavator's bucket contains 90 cubes, each with an edge length of 10~cm, providing fast but coarse particle dynamics. In the HF simulator, the particle's edge length is reduced to 2.5~cm. 
Since the number of particles scales with the inverse cube of the particle size, this yields
$90\times4^3=5760$ particles. In both simulators, the excavator bucket is initialized 1.8~m above the ground and rotated at a fixed angular velocity to execute the pouring motion. After the motion is completed, a depth image of the resulting particle pile is captured and used as the observation for inference. The image is taken from a fixed 3.6~m $\times$ 3.6~m square region centered at a projection of the bucket link onto the ground plane.

Table~\ref{tab:simulation_time} summarizes the computational costs of the two simulators. Generating 1000 LF simulations requires 5 min 50 s, whereas generating 1000 HF simulations requires 16 h 10 min, corresponding to an approximately $166\times$ speedup. Experiments were conducted on a workstation equipped with an Intel Core i9-11900K CPU, 32 GB RAM, and an NVIDIA GeForce RTX 3080 GPU with 10 GB of VRAM.

\subsection{Parameterization and Ground-Truth Settings}

The parameter vector is defined as $\theta = (\theta_f, \theta_{rf}, \theta_{res})$, where $\theta_f$, $\theta_{rf}$, and $\theta_{res}$ denote the static friction, rolling friction, and restitution coefficients of the particles, respectively. These parameters were selected because they govern the dominant particle-level contact behaviors, such as sliding resistance, rolling resistance, and rebound behavior. Motivated by an earlier work \cite{Matl2020-yt} that estimated particle-level parameters in Isaac Gym, we constrain each parameter to $[0,1]$ and assume a uniform prior over this domain during the initial LF exploration stage.

In the sim-to-sim experiments, we consider two synthetic target particles: \textbf{Particle-A} with $\theta^{\mathrm{true}}_A=(0.7, 0.7, 0.3)$ and \textbf{Particle-B} with $\theta^{\mathrm{true}}_B=(0.2, 0.2, 0.3)$. For both cases, the observation is a depth image generated by the HF simulator. The ground-truth parameters are used only for data generation and evaluation, so they are not available during inference.

\subsection{Experimental Setup}

\noindent \textbf{Network architecture:}
 Observations are represented as $250 \times 250$ depth images. NPE uses a CNN-MDN architecture, where a CNN encoder extracts image embeddings and an MDN~\cite{Bishop1994-cl} outputs a $k$-component Gaussian mixture with $k=10$. The bridge model also uses an MDN to model the conditional density over $\Delta$. The CNN encoder consists of two convolutional/max-pooling blocks followed by three fully connected layers with sizes $(120, 128, 128)$ and Tanh activations. All models are trained for 1500 iterations with a batch size of 30.

\begin{table}[tb]
\renewcommand{\arraystretch}{1.3}
\captionsetup{font=footnotesize}
\caption{Wall-clock time required to generate 1000 simulation samples in sim-to-sim experiment. HF simulations were generated using 250 batches of 4 parallel environments, while LF simulations were generated using 10 batches of 100 parallel environments.}
\label{tab:simulation_time}
\centering
\small
\begin{tabular}{ccc}
\toprule
Time at high fidelity & Time at low fidelity & Speedup \\
\midrule
16 h 10 m 45 s & 5 m 50 s & $166\times$ \\
\bottomrule
\end{tabular}
\end{table}

\begin{figure}[tb]
\centering 
\includegraphics[width=0.8\columnwidth]{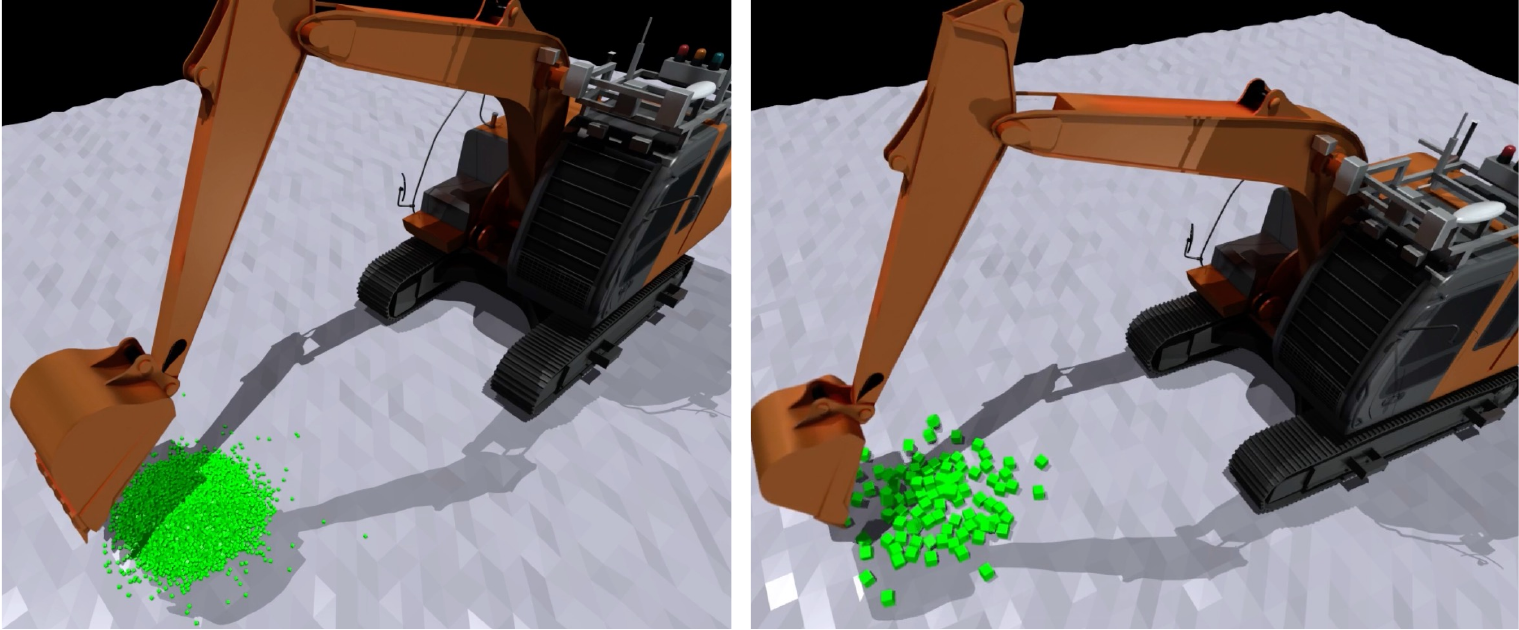}
\captionsetup{font=footnotesize}
\caption{Particle pouring simulation with varying particle sizes. LF simulation uses cubes with 10-cm side length, whereas HF simulation uses cubes with 2.5-cm side length.}
\label{fig:excavator_fidelity}
\vspace{-7pt}
\end{figure}

\vspace{3pt} 
\noindent \textbf{Baseline and cost allocation: }
We compare four methods. 
\begin{itemize}
\item Naive-MF: Sequential multi-fidelity SBI that uses the LF posterior directly as the prior for the HF stage. 
\item Bridged SBI (proposed): Corrected posterior $\qcorr$ obtained via the proposed residual bridge correction. 
\item Bridged w/ Refine (proposed): Optional refined posterior $\qrefined$ built on top of Bridged SBI. 
\item HF-only: NPE trained using only HF simulations.
\end{itemize}
Simulation cost is measured in HF-equivalent units, where one HF simulation has cost $N=1$. 
According to Table~\ref{tab:simulation_time}, 1000 LF simulations correspond to $N=6$. 
For all multi-fidelity methods, the LF posterior is trained with 1000 LF simulations. 
For Bridged w/ Refine, we use 1000 LF simulations for LF posterior estimation and 14 HF simulations for bridge learning, yielding a total pre-refinement cost of $N=20$.

\subsection{Evaluation}
For quantitative evaluation, we train NPE using extensive HF simulations, approximately $N=4000$, and use it as the HF reference posterior, denoted by $p_{\mathrm{ref}}$. This can be considered a reasonable reference, since the aim of multi-fidelity inference is to replace expensive single-fidelity HF inference. We report both $D_{\mathrm{KL}}(p_{\mathrm{ref}}|q)$ and $D_{\mathrm{KL}}(q|p_{\mathrm{ref}})$, where $q$ denotes the estimated posterior. Both divergences are estimated using 300 samples drawn from each posterior and computed via finite summation. These metrics are complementary: Forward KL evaluates coverage of the HF reference posterior, whereas reverse KL evaluates posterior concentration by penalizing probability mass outside the reference high-density region. All reported results are averaged over 10 independent trials.

\begin{figure}[tb]
\centering 
\includegraphics[width=0.80\columnwidth]{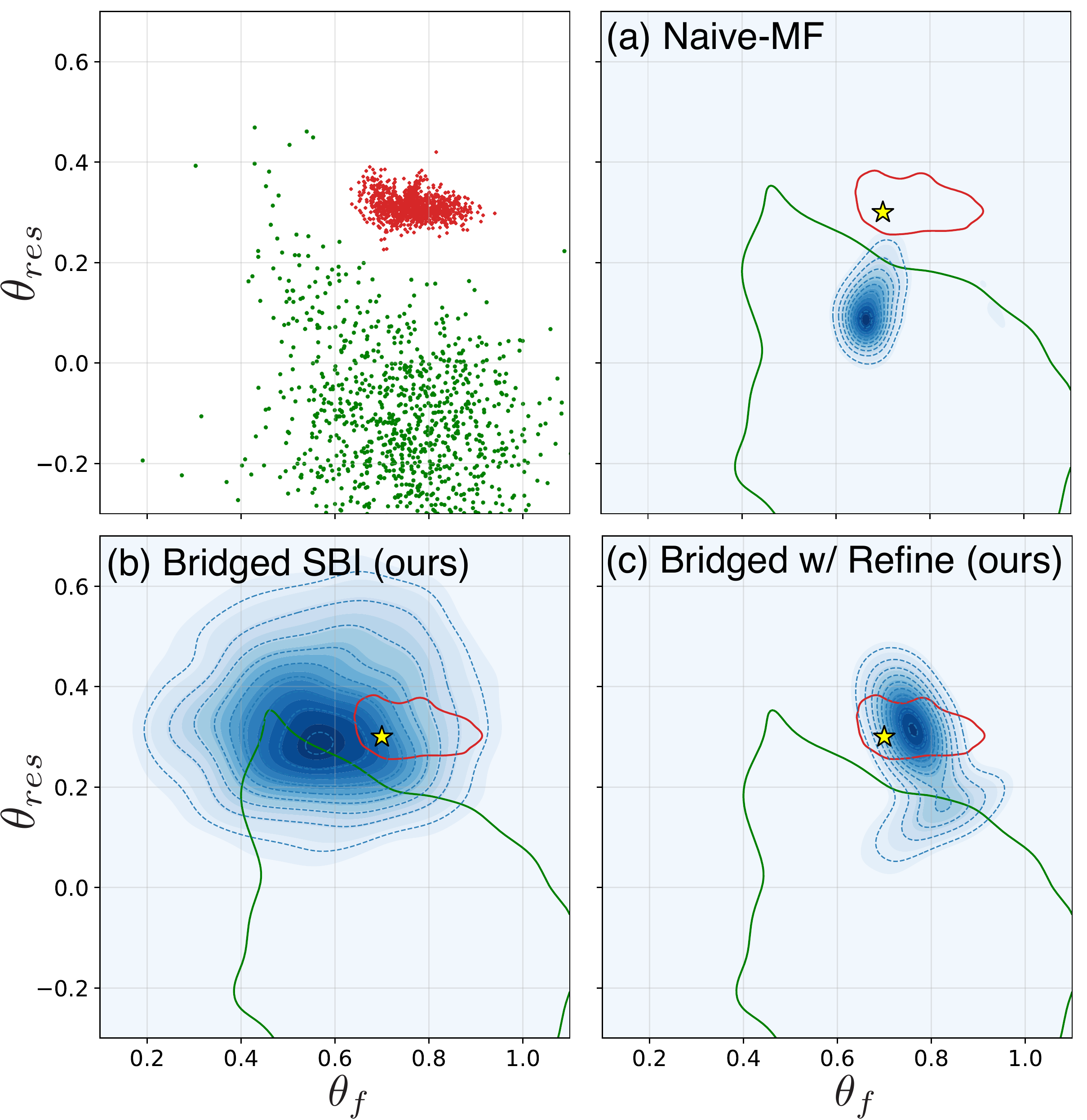}
\captionsetup{font=footnotesize}
\caption{Resulting posteriors for sim-to-sim experiment (Particle-A) at total cost $N=50$, shown as 2D marginals over $(\theta_f, \theta_{res})$. Top-left: Reference HF (red) and LF (green) posterior samples. Solid contours denote $95\%$ high-density regions, and star indicates ground-truth parameter. Naive-MF is biased toward the LF posterior, whereas Bridged SBI recovers HF-consistent regions. Optional refinement further improves posterior concentration.}
\label{fig:s2s_post}
\vspace{-7pt}
\end{figure}

\subsection{Results}

Fig.~\ref{fig:s2s_post} shows the inferred posteriors for Particle-A at a total simulation cost of $N=50$, visualized as two-dimensional marginals over $(\theta_f, \theta_{res})$. The top-left panel presents the reference LF and HF posteriors obtained from large-scale single-fidelity data, and it shows a clear mismatch in the parameter regions. Naive-MF remains confined to the LF-indicated region and fails to recover the HF posterior region containing the ground-truth parameter, providing empirical evidence of LF-induced posterior miscoverage consistent with Proposition~\ref{proposition:naive}. 
In contrast, Bridged SBI recovers posterior mass in the missing HF region, indicating that the bridge model corrects regions that would otherwise be excluded by the LF posterior. 
Bridged w/ Refine further sharpens the corrected posterior while maintaining agreement with the HF reference, highlighting the complementary roles of correction for coverage recovery and refinement for posterior concentration.

Fig.~\ref{fig:s2s_kl_rkl} reports the KL and reverse KL divergence for Particle-A and Particle-B across simulation costs $N$. In both particle types, Naive-MF yields substantially larger KL values than Bridged SBI, indicating that directly using the LF posterior for the HF stage fails to recover important HF-relevant posterior mass. 
In the low-cost regime, Bridged SBI clearly outperforms both Naive-MF and HF-only in terms of KL divergence. 
For example, at $N=7$ for Particle-A, Bridged SBI achieves a mean KL divergence of 4.91, compared with 25.50 for Naive-MF and 24.36 for HF-only, corresponding to an 80.7\% reduction relative to Naive-MF. 
These results show that the residual bridge enables cost-efficient HF posterior inference with reliable coverage of the HF reference posterior under scarce simulation budgets. 
In terms of reverse KL, Bridged SBI remains stable in the low-cost regime, while Bridged w/ Refine substantially reduces reverse KL as additional HF evaluations become available. 
For example, at $N=100$ for Particle-A, Bridged w/ Refine achieves a mean reverse KL of 2.43, whereas Bridged SBI and HF-only are both over 7. 
These results indicate the complementary roles of bridge correction and refinement: Correction recovers posterior mass covering HF parameter candidates, while refinement further sharpens the posterior after coverage has been restored.

\begin{figure}[tb]
    \centering
    \includegraphics[width=0.95\columnwidth]{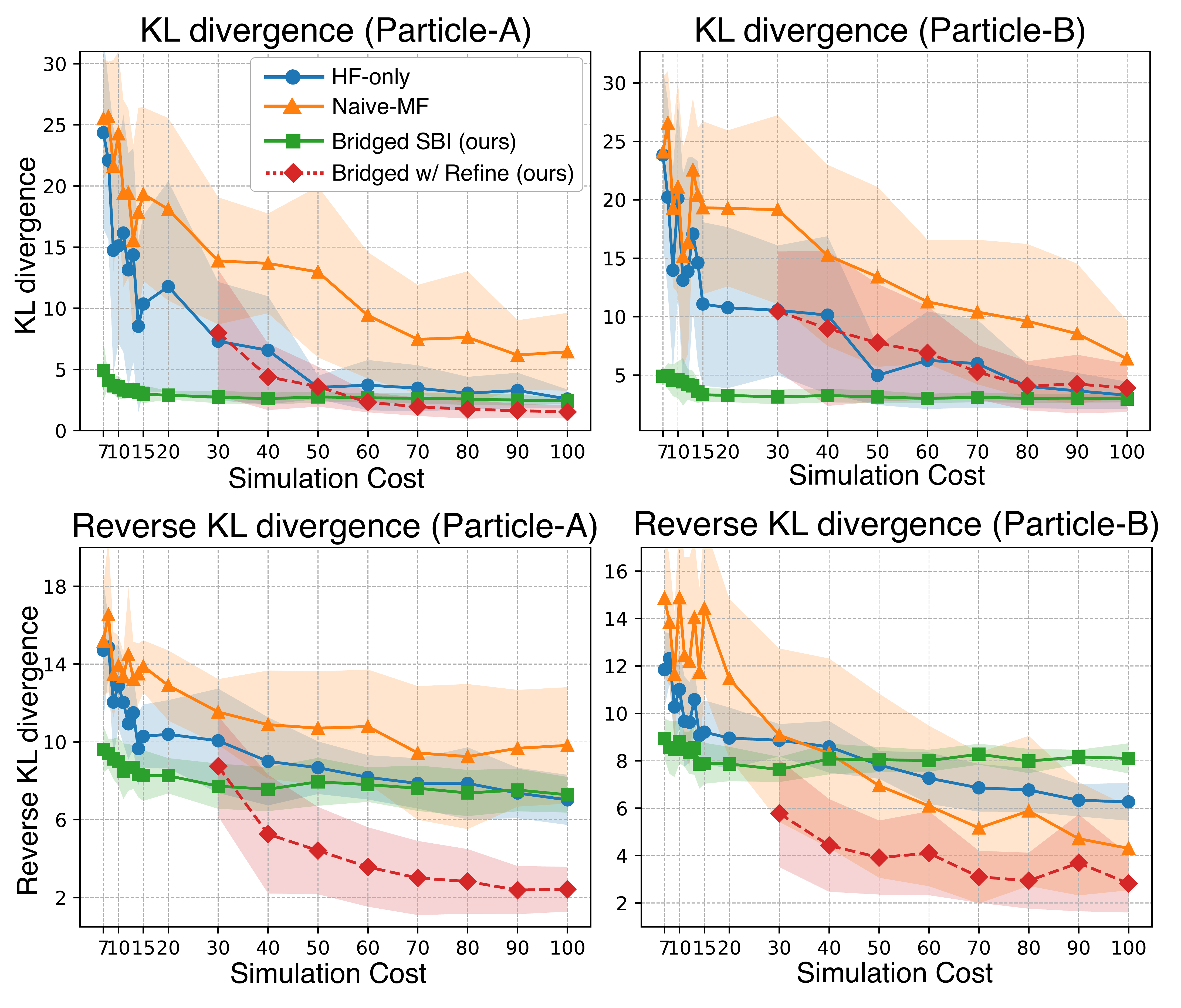}
    \captionsetup{font=footnotesize}
    \caption{KL and reverse KL divergences across simulation costs for sim-to-sim experiments. Upper and lower rows show KL and reverse KL divergences, respectively, for Particle-A and Particle-B. Bridged SBI reduces both divergences compared with HF-only and Naive-MF, especially in low-cost regimes, indicating effective recovery of target-consistent regions under limited budgets. Optional refinement further reduces reverse KL divergence at higher costs, improving posterior concentration.}\label{fig:s2s_kl_rkl}
\vspace{-11pt}
\end{figure}

\section{Real-to-Sim Experiment}

We next evaluate Bridged SBI in a real-to-sim setting, where the goal is to estimate simulator parameters corresponding to the real soil observation.

\subsection{Observation Acquisition and Simulation Model}
The real system (Fig.~\ref{fig:real_excavator}) uses a real excavator corresponding to the one modeled for the simulator used in Section~\ref{sec:s2s}, and it performs the same pouring action. Before pouring, we load 0.09~{\upshape m}$^3$ of soil into the excavator bucket. After pouring, we capture a 3D point cloud of the real soil pile using a Leica BLK2GO scanner. The point cloud is then converted into a mesh and further transformed into a 2D height map, as shown in Fig.~\ref{fig:real_soil_data}, which serves as the observation $\xobs$. The captured area is the same as that used in the sim-to-sim experiments: a square region of 3.6~m $\times$ 3.6~m centered at the bucket link projection on the ground plane. 

We use the same LF and HF simulation settings as in the sim-to-sim experiment, while changing the observation source from simulated to real soil data.

 \begin{figure}[tb]
\centering 
\includegraphics[width=0.75\columnwidth]{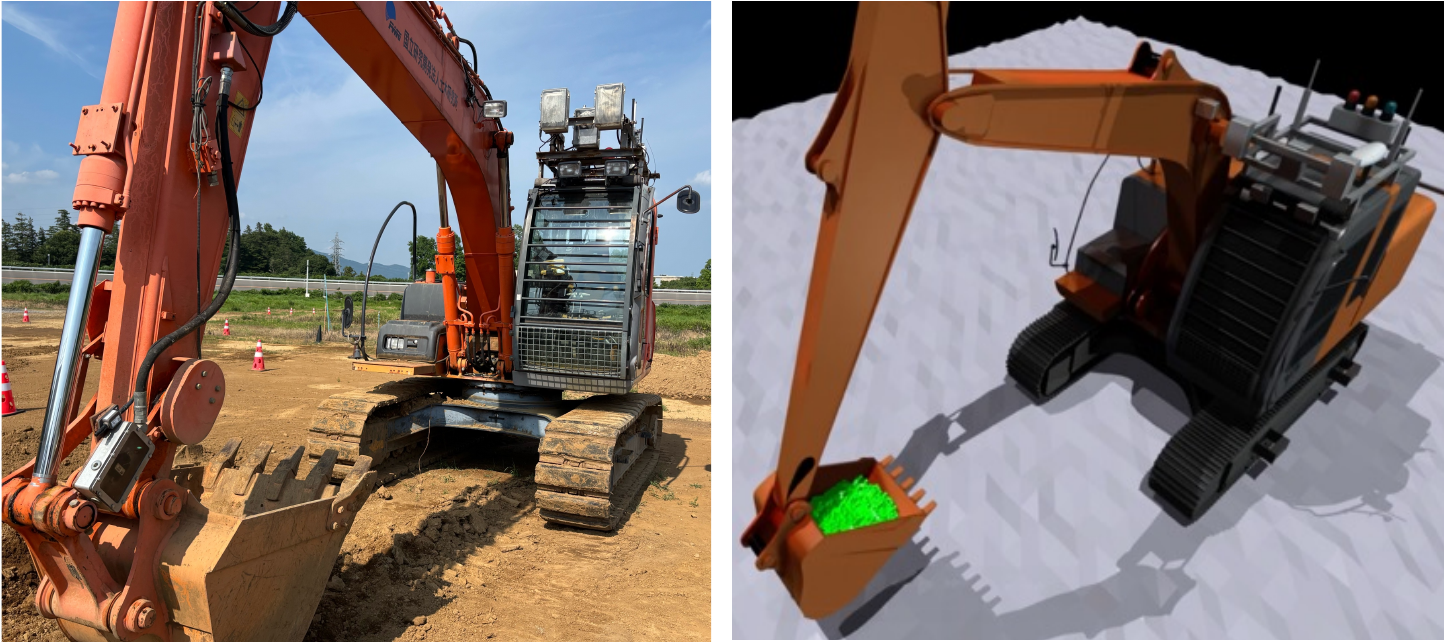}
\captionsetup{font=footnotesize}
\caption{Real and simulated excavator environments. Left: Real excavator performing soil manipulation. Right: Corresponding simulation setup for parameter estimation.}
\label{fig:real_excavator}
\vspace{-7pt}
\end{figure}

\begin{figure}[tb]
\centering 
\includegraphics[width=0.90\columnwidth]{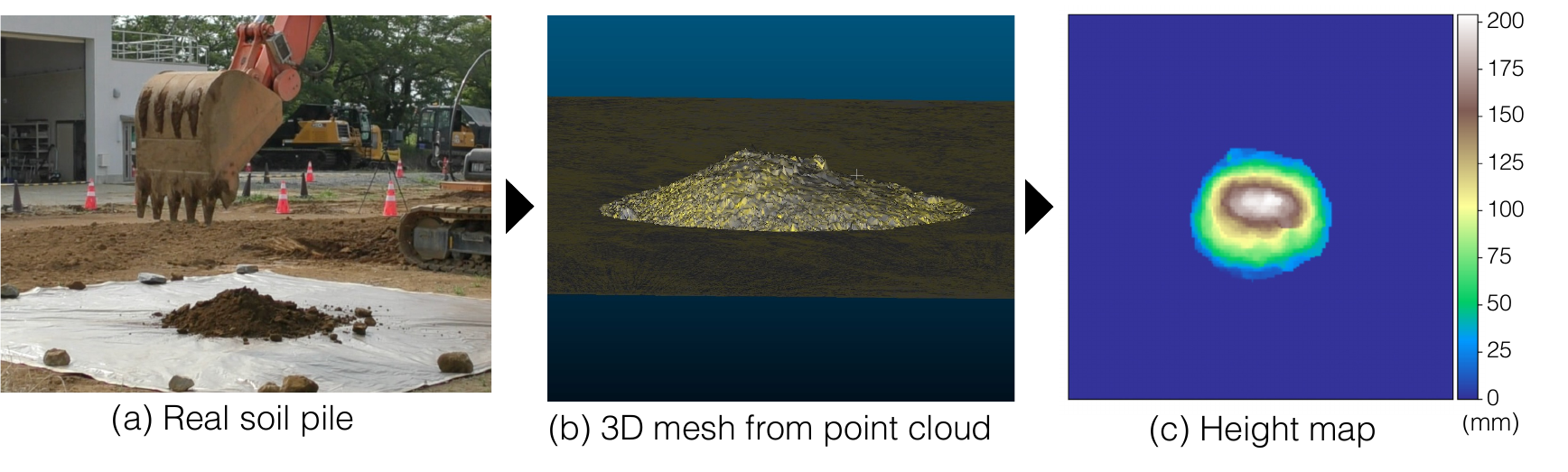}
\captionsetup{font=footnotesize}
\caption{Pipeline of real-to-sim data processing. (a) Soil pouring using an excavator, (b) acquisition of a 3D point cloud and reconstruction of a mesh, and (c) transformation into a 2D height map consistent with the simulation output format.}
\label{fig:real_soil_data}
\vspace{-7pt}
\end{figure}

\subsection{Parameterization}

In addition to the three parameters used in Section~\ref{sec:s2s}, we introduce a parameter to account for the void space between particles. The real soil volume of 0.09~{\upshape m}$^3$ nominally corresponds to 90 cubes with a side length of 10~cm in the LF simulator, or to 5760 cubes with a side length of 2.5~cm in the HF simulator. However, because simulated particle piles contain void space, fewer cubes than these nominal counts are required to reproduce the real soil pile. Therefore, we introduce a scaling parameter $\theta_{num} \in [0.5,1]$ that controls the number of particles, where the lower bound is set to avoid overly sparse particle configurations that no longer preserve the overall pile shape. Accordingly, the LF simulator uses $\theta_{num} \times 90$ cubes, and the HF simulator uses $\theta_{num} \times 5760$ cubes. The full parameter vector is defined as
\[
\theta = (\theta_{f}, \theta_{rf}, \theta_{res}, \theta_{num}),
\]
where $\theta_{f}$, $\theta_{rf}$, and $\theta_{res}$ are defined in Section~\ref{sec:s2s}. The first three parameters lie in $[0,1]$, while $\theta_{num}$ lies in $[0.5,1]$.

\begin{figure}[tb]
\centering 
\includegraphics[width=0.80\columnwidth]{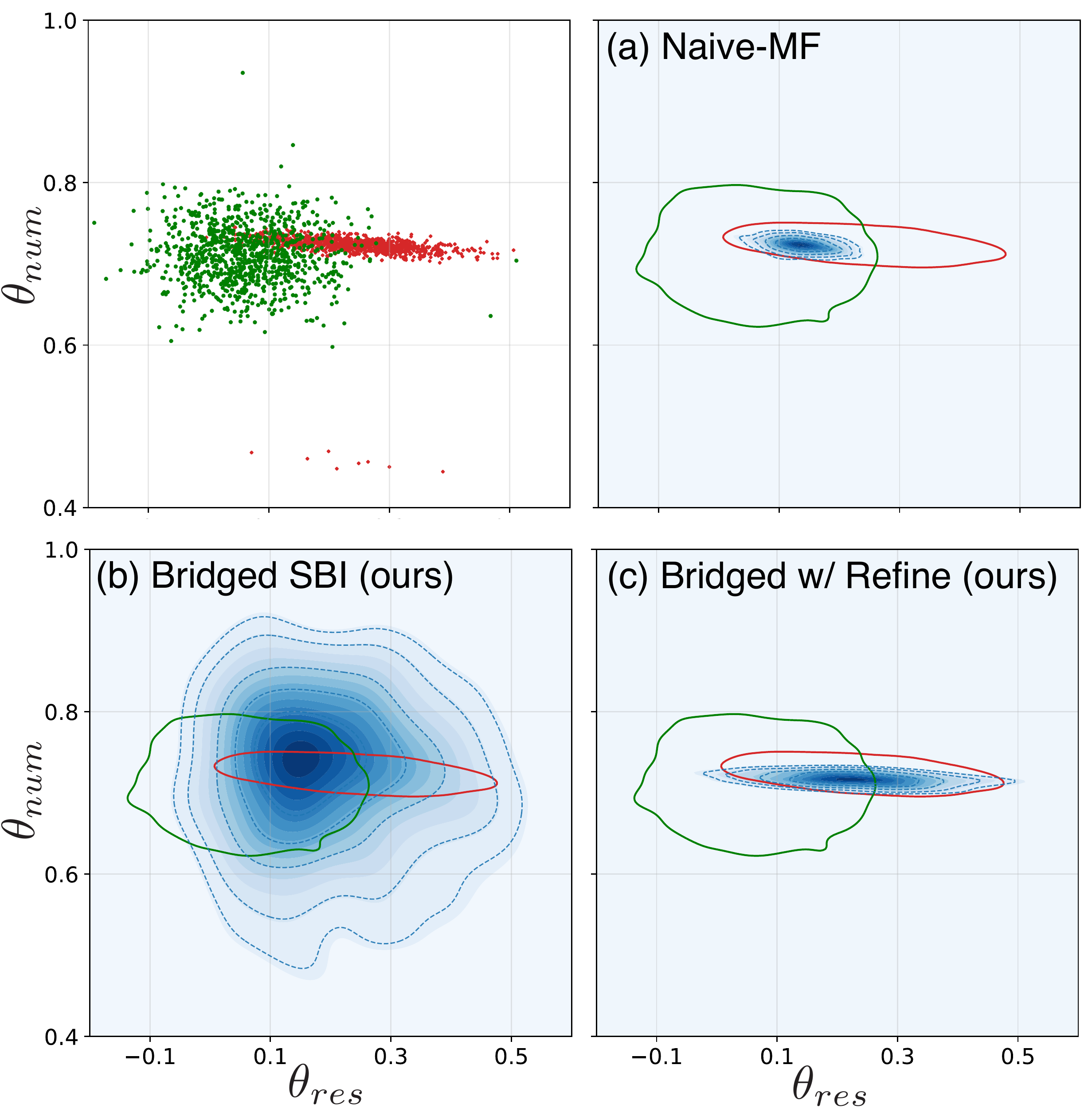}
\captionsetup{font=footnotesize}
\caption{Resulting posteriors for real-to-sim experiment at total simulation cost of $N=50$, shown as 2D marginals over $(\theta_{res}, \theta_{num})$. Top-left: Reference HF (red) and LF (green) posterior samples. Solid contours denote $95\%$ high-density regions. Naive-MF is biased toward LF posterior, whereas Bridged SBI recovers HF-consistent regions. Optional refinement further improves posterior concentration. }
\label{fig:r2s_post}
\vspace{-7pt}
\end{figure}

\subsection{Experimental Setup}
We use the same network architecture, baselines, and budget allocation settings as described in Section~\ref{sec:s2s}.

\subsection{Evaluation}
We train NPE on HF simulations ($N=8000$) and use it as the HF reference posterior. Following the sim-to-sim setting, we estimate forward and reverse KL divergences by finite summation over 300 samples from each posterior.

\subsection{Results}
Fig.~\ref{fig:r2s_post} shows the inferred posteriors at a total simulation cost of $N=50$, visualized as two-dimensional marginals over $(\theta_{res}, \theta_{num})$. The LF and HF reference posteriors exhibit a clear mismatch in the parameter regions. Naive-MF remains constrained to the LF-supported region but fails to recover HF-consistent regions, whereas Bridged SBI restores missing posterior mass and improves alignment with the HF reference. Bridged SBI w/ Refine further sharpens the posterior after correction.

Fig.~\ref{fig:r2s_kl_rkl} reports KL and reverse KL divergence across simulation cost $N$. In the low-budget regime, Bridged SBI achieves the best performance in both metrics. At $N=7$, Bridged SBI obtains a KL divergence of $2.79 \pm 0.77$, compared with $11.19 \pm 6.80$ for Naive-MF and $6.53 \pm 3.62$ for HF-only, corresponding to approximately 75\% and 67\% lower KL values, respectively. As the budget increases, HF-only improves KL and Reverse KL, which is consistent with Bridged SBI being designed primarily for inference in low-budget regimes. When additional HF evaluations are available, Bridged SBI w/ Refine further improves posterior concentration. Overall, the real-to-sim results support our main claim: Reliable posterior inference for particle-parameter calibration requires explicit LF--HF discrepancy correction, rather than treating the LF posterior as a trusted approximation and simply refining it with additional HF simulations.

\section{Discussion} 
The results show that Bridged SBI improves the trade-off between simulation cost and posterior accuracy by using LF simulations for coarse posterior estimation and limited HF simulations for LF--HF discrepancy correction. 

One limitation is that the residual bridge is learned locally around the LF-guided parameter regions. Consequently, the effectiveness of this approach depends on whether a meaningful local LF--HF correction can be learned in this region. If the LF posterior is severely biased or no local cross-fidelity relationship exists, the learned bridge may not fully recover the target posterior. This limitation could be mitigated by increasing exploration around LF posterior samples, using smoothed or broadened sampling distributions~\cite{Giannoukou2025}, or adopting more expressive conditional density models for the residual bridge~\cite{wildberger2023flow}. 

Another limitation is the use of a simple observation projection when sampling LF latent variables. More informative projection models could be used when domain knowledge on the LF and HF observation spaces is available, but manually designing such models may introduce misspecification~\cite{Ward2022-fq, Kennedy2001}. Future work will study how to jointly learn observation projection and parameter correction from simulation data.

\begin{figure}[tb]
\centering 
\includegraphics[width=0.80\columnwidth]{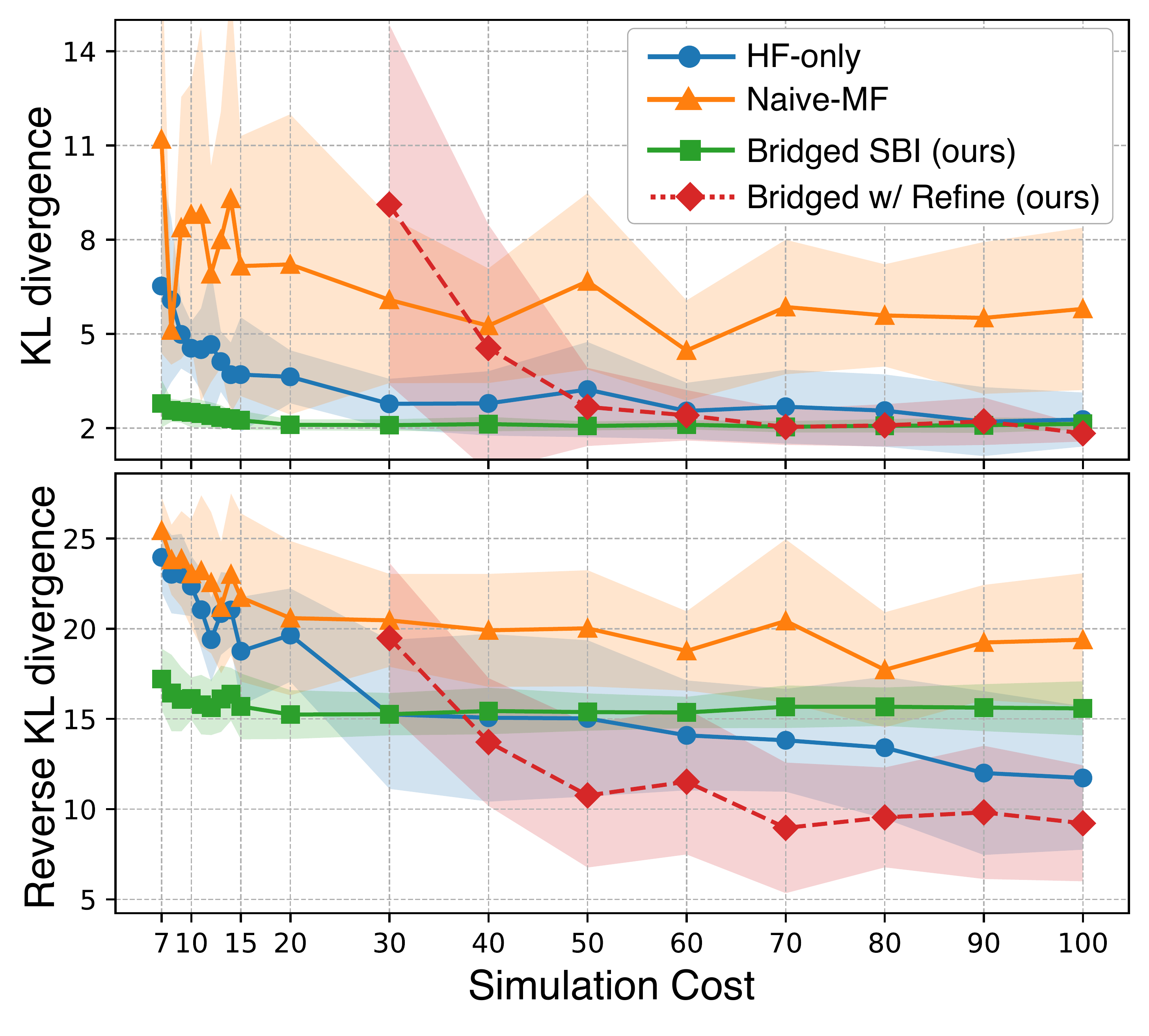}
\captionsetup{font=footnotesize}
\caption{KL and reverse KL divergences across simulation costs for real-to-sim experiment. Bridged SBI achieves lower KL divergence, especially in low-cost regimes, indicating effective recovery of target-consistent regions under limited budgets. Optional refinement further reduces reverse KL divergence at higher costs, improving posterior concentration.}
\label{fig:r2s_kl_rkl}
\vspace{-9pt}
\end{figure}

\section{Conclusion}
In this paper, we proposed Bridged SBI, a cost-efficient calibration framework for particle-based simulators that performs high-fidelity inference through low-fidelity guidance and residual correction. Rather than treating the LF posterior as a trusted guide for sequential refinement, Bridged SBI instead uses it as an informative---albeit biased---guide and then corrects LF posterior samples with a locally learned residual bridge trained from a small number of HF simulations. Our analysis showed that sequential multi-fidelity SBI can suffer from LF-induced posterior miscoverage due to the LF--HF discrepancy. Bridged SBI is designed to alleviate this issue by explicitly modeling the LF--HF discrepancy using the residual bridge. Experiments in sim-to-sim and real-to-sim calibration demonstrated that Bridged SBI could improve HF inference under limited HF simulation budgets.

\bibliographystyle{IEEEtran}
\bibliography{reference}

@article{CUNDALL1979,
  author  = {Cundall, Peter A. and Strack, Otto D. L.},
  title   = {A Discrete Numerical Model for Granular Assemblies},
  journal = {Geotechnique},
  volume  = {29},
  number  = {1},
  pages   = {47--65},
  year    = {1979},
}

@article{COETZEE2017104,
  author  = {Coetzee, C. J.},
  title   = {Review: Calibration of the Discrete Element Method},
  journal = {Powder Technology},
  volume  = {310},
  pages   = {104--142},
  year    = {2017},
}

@article{Cranmer2020-cc,
  author  = {Cranmer, Kyle and Brehmer, Johann and Louppe, Gilles},
  title   = {The Frontier of Simulation-Based Inference},
  journal = {Proceedings of the National Academy of Sciences of the United States of America},
  volume  = {117},
  number  = {48},
  pages   = {30055--30062},
  year    = {2020},
}

@article{Bierwisch2009,
  author  = {Bierwisch, C. and Kraft, T. and Riedel, H. and Moseler, M.},
  title   = {Three-Dimensional Discrete Element Models for the Granular Statics and Dynamics of Powders},
  journal = {Journal of the Mechanics and Physics of Solids},
  volume  = {57},
  number  = {1},
  pages   = {10--31},
  year    = {2009},
  doi     = {10.1016/j.jmps.2008.10.006}
}

@article{dufresne2022energy,
  author  = {Dufresne, Yann and Boulet, Mica{\"e}l and Moreau, St{\'e}phane},
  title   = {Energy dissipation and onset of instabilities in coarse-grained discrete element method on homogeneous cooling systems},
  journal = {Physics of Fluids},
  volume  = {34},
  number  = {3},
  pages   = {033306},
  year    = {2022},
  doi     = {10.1063/5.0081413}
}

@article{ROESSLER201858,
  author  = {Roessler, Thomas and Katterfeld, Andr{\'e}},
  title   = {Scaling of the Angle of Repose Test and Its Influence on the Calibration of {DEM} Parameters Using Upscaled Particles},
  journal = {Powder Technology},
  volume  = {330},
  pages   = {58--66},
  year    = {2018}
}

@ARTICLE{kadokawa2025,
  author={Kadokawa, Yuki and Tahara, Hirotaka and Matsubara, Takamitsu},
  journal={IEEE Transactions on Automation Science and Engineering}, 
  title={Progressive-Resolution Policy Distillation: Leveraging Coarse-Resolution Simulations for Time-Efficient Fine-Resolution Policy Learning}, 
  year={2025},
  volume={22},
  number={},
}

@inproceedings{Ramos2019-oy,
  author    = {Ramos, Fabio and Possas, Rafael and Fox, Dieter},
  title     = {{BayesSim}: Adaptive Domain Randomization via Probabilistic Inference for Robotics Simulators},
  booktitle = {Proceedings of Robotics: Science and Systems},
  year      = {2019},
  doi       = {10.15607/RSS.2019.XV.029}
}

@article{gahee2025-sbi,
  author  = {Kim, Gahee and Matsubara, Takamitsu},
  title   = {{ASBI}: Leveraging Informative Real-World Data for Active Black-Box Simulator Tuning},
  journal = {Applied Intelligence},
  volume  = {55},
  number  = {1028},
  year    = {2025},
}

@inproceedings{Muratore2022-cy,
  author    = {Muratore, Fabio and Gruner, Theo and Wiese, Florian and Belousov, Boris and Gienger, Michael and Peters, Jan},
  title     = {Neural Posterior Domain Randomization},
  booktitle = {Proceedings of the 5th Conference on Robot Learning},
  volume    = {164},
  pages     = {1532--1542},
  year      = {2022}
}

@inproceedings{Papamakarios2016-ki,
  author    = {Papamakarios, George and Murray, Iain},
  title     = {Fast $\epsilon$-Free Inference of Simulation Models with {Bayesian} Conditional Density Estimation},
  booktitle = {Advances in Neural Information Processing Systems},
  volume    = {29},
  pages     = {1028--1036},
  year      = {2016}
}

@inproceedings{Lueckmann2017-sa,
  author    = {Lueckmann, Jan-Matthis and Gon{\c{c}}alves, Pedro J. and Bassetto, Giacomo and {\"O}cal, Kaan and Nonnenmacher, Marcel and Macke, Jakob H.},
  title     = {Flexible Statistical Inference for Mechanistic Models of Neural Dynamics},
  booktitle = {Advances in Neural Information Processing Systems},
  volume    = {30},
  pages     = {1289--1299},
  year      = {2017}
}

@inproceedings{Greenberg2019-on,
  author    = {Greenberg, David and Nonnenmacher, Marcel and Macke, Jakob H.},
  title     = {Automatic Posterior Transformation for Likelihood-Free Inference},
  booktitle = {Proceedings of the 36th International Conference on Machine Learning},
  pages     = {2404--2414},
  year      = {2019}
}

@article{DiRenzo2021CoarseGrainReview,
  author  = {Di Renzo, Alberto and Napolitano, Erasmo Salvatore and Di Maio, Francesco Paolo},
  title   = {Coarse-Grain {DEM} Modelling in Fluidized Bed Simulation: A Review},
  journal = {Processes},
  volume  = {9},
  number  = {2},
  pages   = {279},
  year    = {2021},
  doi     = {10.3390/pr9020279}
}

@article{Warne2018-gk,
  author  = {Warne, David J. and Baker, Ruth E. and Simpson, Matthew J.},
  title   = {Multilevel Rejection Sampling for Approximate {Bayesian} Computation},
  journal = {Computational Statistics \& Data Analysis},
  volume  = {124},
  pages   = {71--86},
  year    = {2018},
  doi     = {10.1016/j.csda.2018.02.009}
}

@article{Prescott2024-re,
  author  = {Prescott, Thomas P. and Warne, David J. and Baker, Ruth E.},
  title   = {Efficient Multifidelity Likelihood-Free {Bayesian} Inference with Adaptive Computational Resource Allocation},
  journal = {Journal of Computational Physics},
  volume  = {496},
  pages   = {112577},
  year    = {2024},
  doi     = {10.1016/j.jcp.2023.112577}
}

@article{Thiele2025-zm,
      title={Simulation-Efficient Cosmological Inference with Multi-Fidelity {SBI}}, 
      author={Leander Thiele and Adrian E. Bayer and Naoya Takeishi},
      year={2025},
      journal = {arXiv preprint arXiv:2507.00514}
}

@article{Hikida2025-ah,
  title   = {Multilevel neural simulation-based inference},
  author  = {Hikida, Yuga and Bharti, Ayush and Jeffrey, Niall and Briol, Fran{\c{c}}ois-Xavier},
  journal = {arXiv preprint arXiv:2506.06087},
  year    = {2025}
}

@inproceedings{Krouglova2026-ew,
  author    = {Krouglova, Anastasia N. and Johnson, Hayden R. and Confavreux, Basile and Deistler, Michael and Gon{\c{c}}alves, Pedro J.},
  title     = {Multifidelity Simulation-Based Inference for Computationally Expensive Simulators},
  booktitle = {International Conference on Learning Representations},
  year      = {2026}
}

@inproceedings{he2016deep,
  title={Deep residual learning for image recognition},
  author={He, Kaiming and Zhang, Xiangyu and Ren, Shaoqing and Sun, Jian},
  booktitle={Proceedings of the IEEE Conference on Computer Vision and Pattern Recognition},
  pages={770--778},
  year={2016}
}

@inproceedings{Ward2022-fq,
  author    = {Ward, Daniel and Cannon, Patrick and Beaumont, Mark and Fasiolo, Matteo and Schmon, Sebastian},
  title     = {Robust Neural Posterior Estimation and Statistical Model Criticism},
  booktitle = {Advances in Neural Information Processing Systems},
  volume    = {35},
  pages     = {33845--33859},
  year      = {2022}
}

@article{Makoviychuk2021IsaacGym,
  title   = {Isaac Gym: High Performance {GPU}-Based Physics Simulation for Robot Learning},
  author  = {Makoviychuk, Viktor and Wawrzyniak, Lukasz and Guo, Yunrong and Lu, Michelle and Storey, Kier and Macklin, Miles and Hoeller, David and Rudin, Nikita and Allshire, Arthur and Handa, Ankur and State, Gavriel},
  journal = {arXiv preprint arXiv:2108.10470},
  year    = {2021}
}

@INPROCEEDINGS{Matl2020-yt,
  author={Matl, Carolyn and Narang, Yashraj and Bajcsy, Ruzena and Ramos, Fabio and Fox, Dieter},
  booktitle={2020 IEEE International Conference on Robotics and Automation (ICRA)}, 
  title={Inferring the Material Properties of Granular Media for Robotic Tasks}, 
  year={2020},
  volume={},
  number={},
  pages={2770-2777},
  doi={10.1109/ICRA40945.2020.9197063}}

@techreport{Bishop1994-cl,
  author      = {Bishop, Christopher M.},
  title       = {Mixture Density Networks},
  institution = {Neural Computing Research Group, Aston University},
  year        = {1994}
}

@article{Giannoukou2025,


  author    = {Giannoukou, Katerina and Marelli, Stefano and Sudret, Bruno},
  title     = {Uncertainty-Aware Multifidelity Surrogate Modeling with Noisy Data},
  journal   = {ASCE-ASME Journal of Risk and Uncertainty in Engineering Systems, Part A: Civil Engineering},
  volume    = {11},
  number    = {3},
  pages     = {04025037},
  year      = {2025},
}

@inproceedings{wildberger2023flow,
  title={Flow Matching for Scalable Simulation-Based Inference},
  author={Wildberger, Jonas and Dax, Maximilian and Buchholz, Simon and Green, Stephen R and Macke, Jakob H and Schölkopf, Bernhard},
  booktitle={Advances in Neural Information Processing Systems},
  volume={36},
  pages={16837--16864},
  year={2023}
}

@article{Kennedy2001,
  author  = {Kennedy, Marc C. and O'Hagan, Anthony},
  title   = {Bayesian Calibration of Computer Models},
  journal = {Journal of the Royal Statistical Society: Series B (Statistical Methodology)},
  volume  = {63},
  number  = {3},
  pages   = {425--464},
  year    = {2001},
  doi     = {10.1111/1467-9868.00294}
}

@book{cover2006elements,
  author    = {Cover, Thomas M. and Thomas, Joy A.},
  title     = {Elements of Information Theory},
  edition   = {2nd},
  publisher = {Wiley-Interscience},
  address   = {Hoboken, NJ, USA},
  year      = {2006}
}

\section*{Appendix A}
\begingroup 
\footnotesize
\setcounter{proposition}{0}
\setcounter{theorem}{0}

\noindent \textbf{Proof of Proposition 1.}
For notational brevity, we write $p_h^\star (\theta | \xobs)$ and $\qnaive (\theta | \xobs)$ as $p_h^\star$ and $\qnaive$, respectively. From (\ref{eq:naive-mf}),
\begin{align*}
    q_{\mathrm{naive}}(A\spaceabs \xobs)
&= Z^{-1}\int_A p_l^\star(\theta\spaceabs \xobs)L_h(\xobs\spaceabs\theta)\,d\theta \\ 
&\le Z^{-1}M\int_A p_l^\star(\theta\spaceabs \xobs)\,d\theta  = C\epsilon,
\end{align*}
where $C:=Z^{-1}M$. Next, decompose the KL divergence over $A$ and its complement $A^c$:
$$
\KL\!\left(p_h^\star \|\,q_{\mathrm{naive}}\right) 
= \int_A p_h^\star
\log \frac{p_h^\star}{q_{\mathrm{naive}}}\,d\theta 
 + \int_{A^c} p_h^\star \log \frac{p_h^\star}{q_{\mathrm{naive}}}\,d\theta.
$$
For the first term, the log-sum inequality~\cite{cover2006elements} gives
$$
\int_A p_h^\star
\log \frac{p_h^\star}{q_{\mathrm{naive}}}\,d\theta 
 \ge p_h^\star(A\spaceabs \xobs)
\log\frac{p_h^\star(A\spaceabs \xobs)}{q_{\mathrm{naive}}(A\spaceabs \xobs)}.
$$
Substituting $p_h^\star(A\spaceabs \xobs)=\rho$ and $q_{\mathrm{naive}}(A\spaceabs \xobs)\le C\epsilon$, we get
\begin{equation*}
\int_A p_h^\star
\log \frac{p_h^\star}{q_{\mathrm{naive}}}\,d\theta
\ge \rho \log\frac{\rho}{C\epsilon}.
\end{equation*}
For the second term, applying the log-sum inequality on $A^c$ yields
$$
\int_{A^c} p_h^\star
\log \frac{p_h^\star}{q_{\mathrm{naive}}}\,d\theta 
\ge p_h^\star(A^c\spaceabs \xobs)
\log \frac{p_h^\star(A^c\spaceabs \xobs)}{q_{\mathrm{naive}}(A^c\spaceabs \xobs)}.
$$
Since $p_h^\star(A^c\spaceabs \xobs)=1-\rho$ and $q_{\mathrm{naive}}(A^c\spaceabs \xobs)\le 1$, we obtain
$$
\int_{A^c} p_h^\star
\log \frac{p_h^\star}{q_{\mathrm{naive}}}\,d\theta 
 \ge (1-\rho)\log\frac{1-\rho}{1} = (1-\rho)\log(1-\rho).
$$
Combining the two bounds gives
$$
\KL\!\left(p_h^\star\,\|\,q_{\mathrm{naive}}\right)  \ge \rho \log \frac{\rho}{C\epsilon} + (1-\rho)\log(1-\rho).
$$

\noindent \textbf{Proof of Proposition 2.}
From the definition, we have 
$$
\qcorr (\theta \spaceabs \xobs) = \int q_\phi (\theta \spaceabs z_l, \xobs) p_l(z_l \spaceabs \xobs)\,dz_l 
= \mathbb{E}_{z_l} [q_\phi ]
$$
Then, we can express 
$$
\KL\!\left(p_h^\star   \,\|\, \qcorr  \right)  = \KL ( \mathbb{E}_{z_l} [K^{\star} ] \, \|\, \mathbb{E}_{z_l} [q_\phi ]).
$$
From the joint convexity of KL divergence and Jensen's inequality, 
$$
    \KL ( \mathbb{E}_{z_l} [K^{\star} ] \, \|\, \mathbb{E}_{z_l} [q_\phi ]) \leq \mathbb{E}_{z_l} [\KL (K^{\star} \,\|\, q_\phi)].
$$
Next, by the definition of $\qcorr$ and $p^{\star}_h$,  and Tonelli's theorem,
\begin{gather*}
\qcorr(A\spaceabs x_{\mathrm{obs}})
= \int q_\phi(A\spaceabs z_l,x_{\mathrm{obs}}) p_l(z_l\spaceabs x_{\mathrm{obs}})\,dz_l,
\\
p_h^\star(A\spaceabs x_{\mathrm{obs}})
= \int K^\star(A\spaceabs z_l,x_{\mathrm{obs}}) p_l(z_l\spaceabs x_{\mathrm{obs}})\,dz_l \nonumber 
= \rho.
\end{gather*}
Therefore,
\begin{align*}
&\left| \qcorr (A\spaceabs x_{\mathrm{obs}}) - \rho \right| \nonumber \\
&= \left| \int \left[ q_\phi(A\spaceabs z_l,x_{\mathrm{obs}}) - K^\star(A\spaceabs z_l,x_{\mathrm{obs}}) \right] p_l(z_l\spaceabs x_{\mathrm{obs}})\,dz_l \right| \nonumber \\
&\leq \int \left| q_\phi(A\spaceabs z_l,x_{\mathrm{obs}}) - K^\star(A\spaceabs z_l,x_{\mathrm{obs}}) \right| p_l(z_l\spaceabs x_{\mathrm{obs}})\,dz_l \nonumber \\
&\leq \eta .
\end{align*}
This implies
\[
q_{\mathrm{bridge}}(A\mid x_{\mathrm{obs}})\ge \rho-\eta.
\]
\endgroup

\end{document}